\documentclass[wcp]{jmlr}

\usepackage[utf8]{inputenc} 
\usepackage{algorithm,algorithmic}
\usepackage[T1]{fontenc}  
\usepackage{hyperref}   
\usepackage{url}       
\usepackage{booktabs}      
\usepackage{amsfonts}   

\usepackage{nicefrac}     
\usepackage{microtype} 
\usepackage{xcolor}      
\usepackage{mathtools}
\usepackage{rotating}
\usepackage{longtable}
\usepackage{natbib}

\makeatletter
\let\Ginclude@graphics\@org@Ginclude@graphics 
\makeatother
\jmlrvolume{189}
\editors{Emtiyaz Khan and Mehmet G\"{o}nen}
\jmlryear{2022}
\jmlrworkshop{ACML 2022}

\title[RoLNiP]{RoLNiP: Robust Learning Using Noisy Pairwise Comparisons}

 \author{\Name{Samartha S Maheshwara} \Email{samartha.s@students.iiit.ac.in}\and
  \Name{Naresh Manwani} \Email{naresh.manwani@iiit.ac.in}\\
   \addr Machine Learning Lab, Kohli Research Block, IIIT Hyderabad, India}

\mathchardef\mhyphen="2D
\def \R {\mathbb{R}}

\def \Y {\mathcal{Y}}

\def \x {\mathbf{x}}
\def \X {\mathcal{X}}

\def \E {\mathbb{E}}

\begin{document}

\maketitle

\begin{abstract}
This paper presents a robust approach for learning from noisy pairwise comparisons. We propose sufficient conditions on the loss function under which the risk minimization framework becomes robust to noise in the pairwise similar dissimilar data. Our approach does not require the knowledge of noise rate in the uniform noise case. In the case of conditional noise, the proposed method depends on the noise rates. For such cases, we offer a provably correct approach for estimating the noise rates. Thus, we propose an end-to-end approach to learning robust classifiers in this setting. We experimentally show that the proposed approach RoLNiP outperforms the robust state-of-the-art methods for learning with noisy pairwise comparisons.
\end{abstract}
\begin{keywords}
Pairwise comparisons, label noise, robust learning, noise rate estimation
\end{keywords}

\section{Introduction}
Standard classifier learning requires pointwise labeled data. However, in many applications, getting pointwise labeled data is not possible. In many applications, the data comes in the pairwise comparison of two examples. For example, whether two samples belong to the same category or not. A frequent use case of this setting is a large-scale annotation framework, e.g., crowd clustering \citep{NIPS2011_c86a7ee3, Yi2012CrowdclusteringWS}.

\citet{pmlr-v80-bao18a} propose an approach that can learn from similar and unlabeled data points. \citet{DBLP:journals/neco/ShimadaBSS21}  propose an empirical risk minimization framework to learn from pairwise similarity data. However, their approach is limited to linear methods. \citep{Hsu2019MulticlassCW} propose a probabilistic framework for learning from pairwise similar, dissimilar data and extend it for deep learning settings.

In practical situations, the pairwise annotations received suffer from noise due to subjective or objective issues. Thus, two examples may belong to the same class and are annotated to be different and vice-versa. 
Figure~\ref{fig:my_label} shows a toy setting of pairwise similar dissimilar data and its noisy version.

\begin{figure}[h]
    \centering
    \includegraphics[scale=0.08]{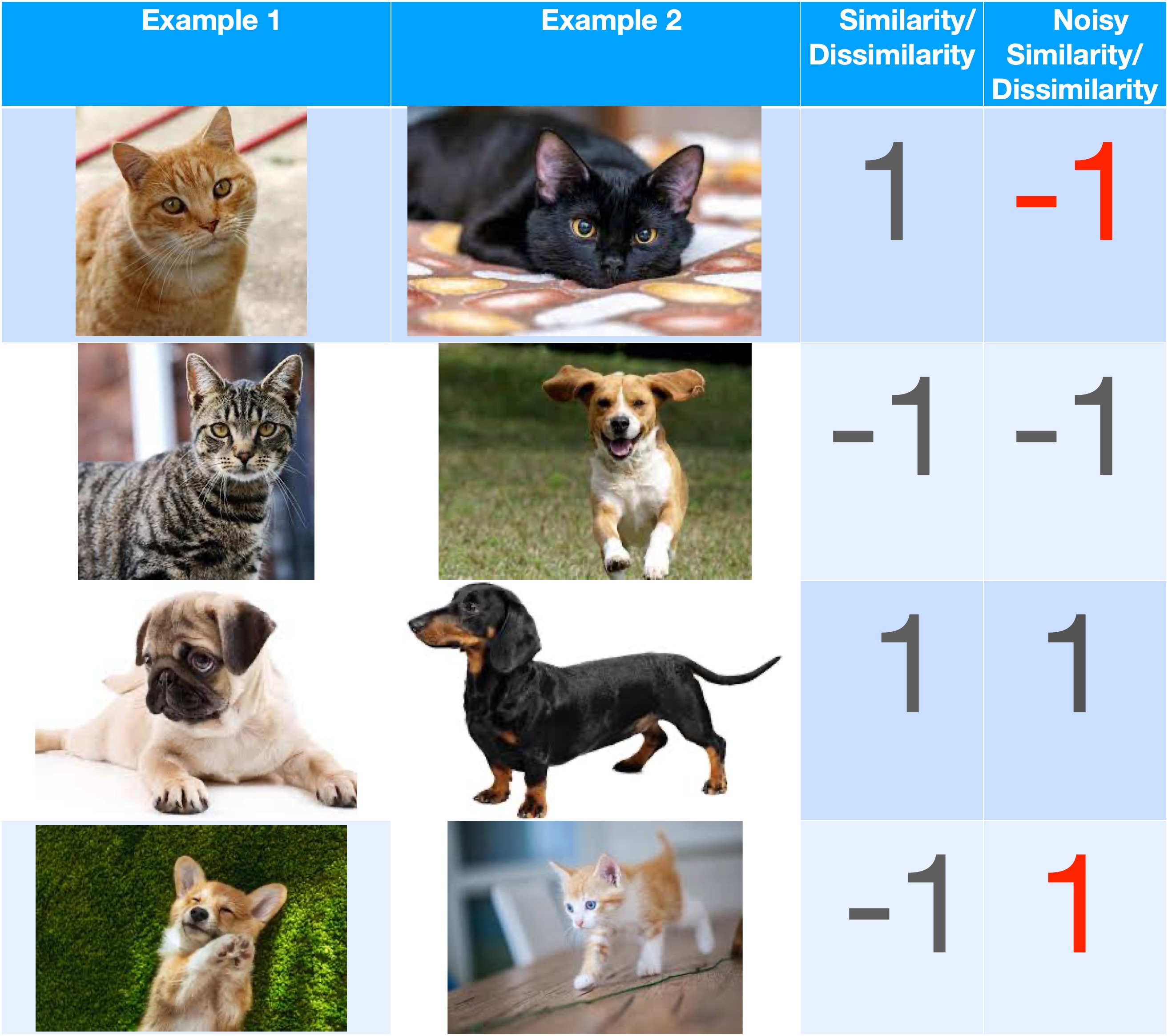}
    \caption{Pairwise Similarity/Dissimilarity Data. The similarity label is +1 when the examples in the pair belong to the same class, then and -1 otherwise. The last column shows the corrupted similarity labels.}
    \label{fig:my_label}
\end{figure}

To deal with this situation, we need robust models for such noisy pairwise comparisons. We quickly review how noisy label setting is dealt with in standard classification problems and borrow some ideas to build robust models for pairwise data.

Label noise in the standard classification problem has been well addressed. \citet{DBLP:journals/tcyb/ManwaniS13} characterized robustness under label noise and shown that convex losses do not lead to robust risk minimization. For binary classification, \citet{10.1016/j.neucom.2014.09.081} show that if a loss function satisfies the condition $l(f(.),1)+l(f(.),-1)=K$ for some finite $K$, then risk minimization becomes robust to label noise under that loss. Loss functions like ramp loss, sigmoid loss, and probit loss are robust against label noise as they satisfy the property mentioned above \citep{10.1016/j.neucom.2014.09.081}.
On the other hand, \citet{NIPS2013_3871bd64} proposes a loss correction-based approach to learning with noisy labels. In this approach, given a loss function, a new loss function is derived such that risk minimization under the new loss is robust to label noise. This approach \citep{NIPS2013_3871bd64} requires noise rate as an input. \citet{Patrini2017MakingDN} proposed a data-dependent approach that can efficiently estimate the noise rates. 

There is not much work in learning binary classifiers using noisy pairwise similar-dissimilar data. The existing approach in this setting \citet{10.1007/978-3-030-86520-7_15} extends the loss correction based approach proposed in \citep{NIPS2013_3871bd64}. However, the applicability of their approach \citep{10.1007/978-3-030-86520-7_15} is limited for the following reasons. (a) Their approach requires noise rate as an input for different cases of SD label noise. (b) They do not provide a provably correct way to estimate the noise rates. (c) Their experimental results are limited to squared error loss only. 

This paper proposes a new approach for learning robust classifiers using noisy pairwise similar dissimilar data. The proposed approach does not change the loss function at all. Instead, it is based on the principle of robust loss functions. We propose the conditions under which risk minimization becomes robust to the SD label noise pairwise similar dissimilar data. We use the acronym {\it RoLNiP} ({\bf Ro}bust {\bf L}earning using {\bf N}o{\bf i}sy {\bf P}airwise comparisons) for our approach. We make the following contributions to this work.
\subsubsection*{Our Contributions:}
\begin{enumerate}
    \item We derive sufficiency conditions under which the risk minimization framework becomes robust to SD label noise in similar dissimilar data. 
    \begin{itemize}
        \item {\bf Uniform SD Label Noise:} In this case, if the loss function satisfies the condition $l(f(.),1)+l(f(.),-1)=K, \forall f\in \mathcal{F}$ for some finite $K$, then we show that the risk minimization becomes robust as long as the uniform noise rate is less than 0.5. RoLNiP does not require noise rate as an input for the uniform noise case. In contrast, the loss correction-based approach \citep{DBLP:journals/neco/ShimadaBSS21} requires the noise rate to learn a robust classifier.
        \item {\bf SD Conditional SD Label Noise: } In this case, if the loss function satisfies the condition $l(f(.),1)+l(f(.),-1)=K, \forall f\in \mathcal{F}$ for some finite $K$, then we can scale the loss function $(l(.,.)$ such that the scaled loss function becomes robust to the label noise. It also requires that the sum of noise rates should be less than 1. Note that our approach also needs noise rate as input for this case.
    \end{itemize} 
    \item We propose a provably correct approach for learning the noise rates for SD conditional SD label noise. We show that this approach efficiently learns the noise rates for different settings.
        \item We experimentally show the robustness of RoLNiP with different noise rates and with different datasets. 
\end{enumerate}

\section{Learning with Pairwise Similarity/Dissimilarity Data}
Here, we are interested in learning only binary classifiers using pairwise similarity/dissimilarity (SD) data. Let $\X\subseteq \R^d$ be the instance space from which the feature vectors are generated. Let $\Y=\{+1,-1\}$ be the label space. Let $S = \{(\x_1, \x_1', \tau_1), (\x_2, \x_2', \tau_2), \ldots,(\x_N, \x'_N, \tau_N)\}$ be the noise-free SD training data, where $(\x_i,\x_i')\in \X\times \X,\forall i\in [N]$. $\tau_1,\ldots,\tau_N$ are the SD labels. $\tau_i=1$ if $y_i=y_i'$ ($\x_i$ and $\x_i'$ belong to the same class). Similarly, $\tau_i=-1$ if $y_i\neq y_i'$ ($\x_i$ and $\x_i'$ belong to different classes). Let $\pi_+=p(y=+1)$ and $\pi_-=p(y=-1)$ be the class prior probabilities such that $\pi_++\pi_-=1$. Let $\pi_S=p(\tau=1)$ and $\pi_D=p(\tau=-1)$.  

\subsection{Risk Minimization Framework for Learning with SD Labels} We use deep neural networks to learn the classifiers. Let $\mathcal{F}$ be the class of real-valued classifier functions represented by deep neural networks. Any $f\in \mathcal{F}$ ($f:\mathcal{X}\rightarrow \mathbb{R}$) denotes a real-valued binary classifier corresponding to a particular network. To capture the difference between the value predicted by the network and the actual label, we use the loss function $l:\mathbb{R}\times \mathcal{Y}\rightarrow \mathbb{R}_+$. The standard risk minimization framework finds the classifier by minimizing the expected value of loss $l(.,.)$. 

In the SD label setting,  
\citep{DBLP:journals/neco/ShimadaBSS21} propose a risk minimization framework for learning binary classifiers with SD risk. The risk of a classifier $f(.)$ can be rewritten in terms of SD labels as follows \citep{DBLP:journals/neco/ShimadaBSS21}.
\begin{align}
\label{eq:noise-free-sd-risk}
    R_{L}(f) = \pi_S \E_{ p_{S}(\x, \x')} \bigg[\frac{L(f(\x), 1) + L(f(\x'), 1)}{2}\bigg] + \pi_D \E_{ p_{D}(\x, \x')} \bigg[\frac{L(f(\x), -1) + L(f(\x'), -1)}{2}\bigg]
\end{align}
where $L(z, t) = \frac{\pi_+}{\pi_+ - \pi_-} l(z, t) - \frac{\pi_-}{\pi_+ - \pi_-} l(z, -t)$. Here, $p_S(\x,\x')=p(\x,\x'|y=y')=p(\x,\x'|\tau=1)$ and $p_D(\x,\x')=p(\x,\x'|y\neq y')=p(\x,\x'|\tau=-1)$.

\subsection{SD Label Noise} 
Here, we introduce SD label noise which affects SD label $\tau$. In the presence of SD label noise, instead of $\tau$, we observe the noisy SD label $\Tilde{\tau}$. $\Tilde{\tau}$ can be either same as $\tau$ or $-\tau$ (depending on the underlying noise process, which could be deterministic or stochastic). One simplest stochastic SD label noise model is as follows. The noise corrupted SD label $ \Tilde {\tau}_i$ is same as noise-free SD label $\tau_i$ with probability $1-\eta_i$. $ \Tilde{\tau}_i = -\tau_i $ with probability $\eta_i$. In this paper, we treat two kinds of stochastic SD label noise.
\begin{enumerate}
    \item {\bf Uniform SD Label Noise:} In this case, $P(\Tilde{\tau}_i=\tau_i)=1-\eta,\;\forall i\in[N]$. Thus, $\eta_i=\eta,\;\forall i\in [N]$. Thus for every example pair $(\x_i,\x_i')$, noise affects $\tau_i$ with the same rate $\eta$. 
    \item {\bf SD Conditional SD Label Noise:} In this case, the noise rate depends on the value of the noise-free SD label. Thus, there can be two cases here. (a) In the first case, we assume that $\tau=1$. Thus, the probability that the noisy SD label $\Tilde{\tau}$ takes value -1 is $\eta_S$ (i.e., $P(\Tilde{\tau}=-1|\tau=1)=\eta_S$). (b) In the second case, we assume that $\tau=-1$. Thus, the probability that the noisy SD label $\Tilde{\tau}$ takes value 1 is $\eta_D$ (i.e., $P(\Tilde{\tau}=1|\tau=-1)=\eta_D$).
\end{enumerate}

\section{Robust Learning in Presence of SD Label Noise} Consider the noisy pairwise SD training data as $ S^{\eta}_{SD} = \{(\x_1, \x_1', \Tilde{\tau}_1), \ldots,(\x_N, \x'_N, \Tilde{\tau}_N)\}$, where $ \Tilde {\tau}_i$ is the noise corrupted SD label for example pair $(\x_i,\x_i')$. This paper aims to learn the true underlying classifier using the noisy pairwise similarity data.  

\subsection{How Is It Different from Learning With Label Noise?}
Note that the setting discussed here is different from learning with label noise. Learning actual similarity functions using noisy pairwise similarity labels is a problem analogous to learning with noisy labels.

However, we are trying to learn the underlying classifier using the noisy pairwise similarity labels, which is a more complex problem than learning the actual similarity function.

\subsection{Robustness Criteria}
Here, the noise affects the pairwise similarity labels. But, it eventually can cause label noise.\footnote{For example, $\x_i$ and $\x_j$ are two feature vectors such that $\x_i$ belong to the positive class. We don't know about the class of $\x_j$. But, we know the noisy pairwise similarity between $\x_i$ and $\x_j$, that is, $\Tilde{\tau}=-1$. If $\Tilde{\tau}$ is the same as the actual similarity, then the true class of $\x_j$ is negative. On the other hand, if $\Tilde{\tau}$ is the opposite of the actual similarity, then the true similarity is +1, and the actual class of $\x_j$ is positive. Thus, we see that the noise in the pairwise similarity also induces noise in the labels.} Therefore, achieving robustness in predicting pairwise similarity does not imply that the model is robust to the label noise (which is caused by noisy pairwise similarity labels). Thus, a better robustness goal is to ensure that the noise in pairwise similarity labels does not affect label predictions. Which means, classifier learning should be robust to noise in pairwise similarity labels.

Therefore, we use the noise tolerance definition proposed in \citep{DBLP:journals/tcyb/ManwaniS13} to ensure the robustness in label predictions. Let $f^*$ be the classifier learnt using an algorithm $\mathcal{A}$ with noise-free SD training data. Also, let $f_{\eta}^*$ be the classifier learnt using algorithm $\mathcal{A} $ with noisy SD training data. We say that the algorithm $\mathcal{A}$ is robust to SD label noise if
both $f^*$ and $f_{\eta}^*$ achieve the same accuracy on noise-free data as follows. In other words,
\begin{equation}
    P_{\mathcal{X}\times \mathcal{Y}}(f^*(\x)=y)=P_{\X\times \Y}(f^*_{\eta}(\x)=y).
\end{equation}

\subsection{Robustness of Risk Minimization With SD Label Noise} 
Here, we establish conditions under which risk minimization becomes robust in the presence of SD label noise. Here the probability is computed with respect to the joint distribution on $\X\times \Y$.
\begin{theorem}
Let $ \eta < 0.5 $ and let the loss function $ l $ satisfy $ l(f(\x), 1) + l(f(\x), -1) = K,\;\forall \x,\;\forall f $ for some positive constant K. Then SD risk minimization becomes noise tolerant to uniform noise.
\end{theorem}
\begin{proof}
SD risk $ R_{SD}$ under noise-free case is described in eq.(\ref{eq:noise-free-sd-risk}).
SD risk in presence of noise ($ R_{L}^{\eta}(f) $) is described as follows.

\begin{equation}
    \begin{split}
        &R_{L}^{\eta}(f)  = \pi_S^{\eta} \left( \eta \E_{ p_{D}} \left[ \frac{L (f(\x), 1) + L (f(\x'), 1)}{2} \right]  + (1-\eta) \E_{ p_{S}} \left[\frac{L(f(\x), 1) + L(f(\x'), 1)}{2} \right] \right)\\
        & + \pi_D^{\eta} \left( \eta \E_{ p_{S} } \left[ \frac{L (f(\x), -1) + L (f(\x'), -1)}{2} \right]  + (1-\eta) \E_{ p_{D}} \left[\frac{L(f(\x), -1) + L(f(\x'), -1)}{2} \right] \right)
    \end{split}
\end{equation}
Let $g(t) = \frac{L (f(\x), t) + L (f(\x'), t)}{2}$. Then,
\begin{equation} \label{RSDeta}
    \begin{split}
        R_{L}^{\eta}(f) & = 
         \pi_S^{\eta} \left( \eta \E_{p_{D}} [g(1) ] + (1-\eta) \E_{p_{S}} [g(1) ] \right) + \pi_D^{\eta} \left( \eta \E_{p_{S}} [ g(-1) ] + (1-\eta) \E_{p_{D}} [g(-1) ] \right)\\
         &=\eta \pi_D^{\eta} \E_{p_{S}} [ g(-1) ] + (1-\eta) \pi_S^{\eta} \E_{p_{S}} [ g(1) ] + \eta \pi_S^{\eta} \E_{p_{D}} [g(1)] + (1-\eta) \pi_D^{\eta} \E_{p_{D}} [ g(-1) ]
    \end{split}
\end{equation}
Using the assumption $l(f(\x), 1) + l(f(\x), -1) = K$, we see that 
\begin{equation*}
    \begin{split}
        &L(f(\x), 1) + L(f(\x), -1)=  \frac{\pi_+l(f(\x), 1)}{\pi_+ - \pi_-}  - \frac{\pi_-l(f(\x), -1)}{\pi_+ - \pi_-}  +  \frac{\pi_+l(f(\x), -1)}{\pi_+ - \pi_-}  - \frac{\pi_-l(f(\x), 1)}{\pi_+ - \pi_-}  \\
        & \quad \quad\quad \quad = \bigg( \frac{\pi_+ - \pi_-}{\pi_+ - \pi_-} \bigg) l(f(\x), 1) + \bigg( \frac{\pi_+ - \pi_-}{\pi_+ - \pi_-} \bigg) l(f(\x), -1) = l(f(\x), 1) + l(f(\x), -1)=K.
    \end{split}
\end{equation*}
Thus, $g(1) + g(-1)  =   \frac{L(f(\x), 1) + L(f(\x'), 1)}{2}  +  \frac{2K -  L(f(\x), 1) - L(f(\x'), 1)}{2}  = K$. Using this, we get,
\begin{equation*} \label{interRSD}
    \begin{split}
        R_{L}^{\eta}(f) 
        & = \eta \pi_D^{\eta} \E_{p_{S}} [ K - g(1) ] + (1-\eta) \pi_S^{\eta} \E_{p_{S}} [g(1)]
        + \eta \pi_S^{\eta} \E_{p_{D}} [K - g(-1) ] + (1-\eta) \pi_D^{\eta} \E_{p_{D}} [g(-1)]\\
        & = \eta \pi_D^{\eta} K + \eta \pi_S^{\eta} K + (1-\eta) \pi_S^{\eta} \E_{p_{S}} [g(1)]- \eta \pi_D^{\eta} \E_{p_{S}} [g(1)]
         + (1-\eta) \pi_D^{\eta} \E_{p_{D} } [ g(-1)]\\
         & \quad- \eta \pi_S^{\eta} \E_{p_{D}} [g(-1)]\\
        & = \eta \pi_D^{\eta} K + \eta \pi_S^{\eta} K + [(1-\eta) \pi_S^{\eta} - \eta \pi_D^{\eta}] \E_{ p_{S}} [g(1) ] + [(1-\eta) \pi_D^{\eta} - \eta \pi_S^{\eta}] \E_{p_{D}} \big{[} g(-1) \big{]} \\
        & = \eta K + \bigg{[} \frac{(1-\eta) \pi_S^{\eta} - \eta \pi_D^{\eta}}{\pi_S} \bigg{]} \pi_S \E_{p_{S}} \big{[} g(1) \big{]} + \bigg{[} \frac{(1-\eta) \pi_D^{\eta} - \eta \pi_S^{\eta}}{\pi_D} \bigg{]} \pi_D \E_{p_{D}} \big{[} g(-1) \big{]} 
    \end{split}
\end{equation*}

We can easily see that $\frac{(1-\eta) \pi_S^{\eta} - \eta \pi_D^{\eta}}{\pi_S}= \frac{(1-\eta) \pi_D^{\eta} - \eta \pi_S^{\eta}}{\pi_D}=1-2\eta$. 
Using this in eq.(\ref{interRSD}), we get $
        R_{L}^{\eta}(f)   = \eta K + \big{[} 1 - 2\eta \big{]} \pi_S \E_{p_{S}} \big{[} g(1) \big{]} + \big{[} 1 - 2\eta \big{]} \pi_D \E_{p_{D}} \big{[} g(-1) \big{]}
        \
         = \eta K + \big{[} 1-2\eta \big{]} R_{SD}(f)$.
Let $f^*$ be the minimizer of SD risk $R_{L}(f)$, then $R_{L}(f^*)\leq R_{L}(f),\;\forall f\in \mathcal{F}$. But we also know that $\eta < 0.5$. Thus, we see that $\forall f \in \mathcal{F}$, $R_{L}^{\eta}(f^{*}) - R_{L}^{\eta}(f)  = (1 - 2\eta ) \big{[} R_{L}(f^{*}) - R_{L}(f) \big{]}\leq 0$. Therefore, $f^*$ remains minimizer of the noisy SD risk $ R_{L}^{\eta} (f) $ also. Thus, risk minimization is robust to uniform SD label noise.
\end{proof}
Above theorem gives a generic recipe to make the risk minimization robust to the uniform SD label noise by choosing a loss functions $l(.,.)$ that satisfies the condition $l(f(\x),1)+l(f(\x),-1)=K,\;\forall \x,\;\forall f$ for some positive constant $K$. Note that this method works irrespective of the noise rate as long as it is less than $0.5$. It is also worth noting that this approach does not require knowing the noise rate, unlike the loss correction based approach \citep{10.1007/978-3-030-86520-7_15}.

\section{Robust Learning in the presence of SD Conditional Noise}
Here, we develop an approach to learning robust classifiers in the presence of SD conditional noise. Unlike uniform noise, this case needs a different approach to achieve robustness, as described in the following theorem. 
\begin{theorem}
Let $\eta_S=P(\Tilde{\tau}=-1|\tau=1)$ and $\eta_D=P(\Tilde{\tau}=1|\tau=-1)$ such that $\eta_S+\eta_D<1$. Let the loss function $ l $ be such that $ l(f(\x), 1) + l(f(\x), -1) = K,\;\forall \x,\;\forall f $ for some positive constant $K$. Let $L(f(\x), t) = \frac{\pi_+}{\pi_+ - \pi_-} l(f(\x), t) - \frac{\pi_-}{\pi_+ - \pi_-} l(f(\x), -t)$. SD risk $ R_{L}(f) $ under noise free case as described in eq.(\ref{eq:noise-free-sd-risk}).
Let loss function $\hat{L}$ be such that $\hat{L}(f(\x),1)=L(f(\x),1)$ and $\hat{L}(f(\x),-1)=CL(f(\x),-1)$ where $C= \frac{\pi_S^{\eta} \big[ (1-\eta_S)\pi_D + \eta_D\pi_S \big]}{\pi_D^{\eta} \big[ \eta_S \pi_D + (1-\eta_D)\pi_S \big] }$. Let $R_{\hat{L}}^{\eta}$ denotes the risk with loss function $\hat{L}$ under SD conditional noise. Then, the minimizer of the noisy risk $R_{\hat{L}}^{\eta}$ is the same as the minimizer of noise-free risk $R_{{L}}$.  
\end{theorem}
\begin{proof}
The risk with loss function $\hat{L}$ under SD conditional noise is as follows. 
\begin{equation*}
    \begin{split}
        &R_{\hat{L}}^{\eta}(f)  = \pi_S^{\eta} \left[ \eta_D \E_{p_{D}} \left[ \frac{\hat{L} (f(\x), 1) + \hat{L} (f(\x'), 1)}{2} \right] 
         + (1-\eta_S) \E_{p_{S}} \left[\frac{\hat{L}(f(\x), 1) + \hat{L}(f(\x'), 1)}{2} \right] \right]\\
        &+ \pi_D^{\eta} \left[ \eta_S \E_{p_{S}} \left[ \frac{\hat{L} (f(\x), -1) + \hat{L} (f(\x'), -1)}{2} \right]
        + (1-\eta_D) \E_{p_{D}} \left[ \frac{\hat{L}(f(\x), -1) + \hat{L}(f(\x'), -1)}{2} \right] \right]
    \end{split}
\end{equation*}
Let $g(t) =  \frac{L (f(\x), t) + L (f(\x'), t)}{2} $ and $g'(t)  = \frac{\hat{L} (f(\x), t) + \hat{L} (f(\x'), t)}{2} $. Using the fact that $\hat{L}(z, 1)  = L(z, 1)$ and $\hat{L}(z, -1)  = CL(z, -1)$, we see that $g'(1)  = \frac{\hat{L}(f(\x), 1) + \hat{L}(f(\x'), 1)}{2} = \frac{L(f(\x), 1) + L(f(\x'), 1)}{2} = g(1)$ and $g'(-1)  = \frac{\hat{L}(f(\x), -1) + \hat{L}(f(\x'), -1)}{2} = \frac{CL(f(\x), 1) + CL(f(\x'), 1)}{2} = Cg(-1)$. Using this, we get the following.

\begin{equation*}
    \begin{split}
        R_{\hat{L}}^{\eta}(f)  &= \eta_S \pi_D^{\eta} \E_{p_{S} } \big{[} g'(-1) \big{]} + (1-\eta_S) \pi_S^{\eta} \E_{ p_{S} } \big{[} g'(1) \big{]}
        + \eta_D \pi_S^{\eta} \E_{ p_{D} } \big{[} g'(1) \big{]} + (1-\eta_D) \pi_D^{\eta} \E_{ p_{D} } \big{[} g'(-1) \big{]}\\
        \
        & = \eta_S \pi_D^{\eta} \E_{ p_{S} } \big{[} Cg(-1) \big{]} + (1-\eta_S) \pi_S^{\eta} \E_{ p_{S} } \big{[} g(1) \big{]} + \eta_D \pi_S^{\eta} \E_{ p_{D} } \big{[} g(1) \big{]} + (1-\eta_D) \pi_D^{\eta} \E_{ p_{D} } \big{[} Cg(-1) \big{]}
        \end{split}
        \end{equation*}
 Using $l(f(\x), 1) + l(f(\x), -1) = K$, we notice that $L(f(\x), 1) + L(f(\x), -1) =K$. Thus, $g(1) + g(-1)  =  \frac{L(f(\x), 1) + L(f(\x'), 1)}{2}  +  \frac{L(f(\x), -1) + L(f(\x'), -1)}{2} =  \frac{L(f(\x), 1) + L(f(\x'), 1)}{2} +  \frac{2K - [ L(f(\x), 1) + L(f(\x'), 1)] }{2}  = K$. Using this, we get the following.
        \begin{equation*}
            \begin{split}
    &R_{\hat{L}}^{\eta}(f) = \eta_S \pi_D^{\eta} C \E_{ p_{S} } \big{[} K - g(1) \big{]} + (1-\eta_S) \pi_S^{\eta} \E_{ p_{S} } \big{[} g(1) \big{]}\\
        &\ \ \ + \eta_D \pi_S^{\eta} \E_{ p_{D} } \big{[} K - g(-1) \big{]} + (1-\eta_D) \pi_D^{\eta} C \E_{ p_{D} } \big{[} g(-1) \big{]}\\
        & = \eta_S \pi_D^{\eta} K C + \eta_D \pi_S^{\eta} K + \big[ (1-\eta_S) \pi_S^{\eta} - \eta_S \pi_D^{\eta} C \big] \E_{ p_{S} } \big{[} g(1) \big{]} + \big[ (1-\eta_D) \pi_D^{\eta} C - \eta_D \pi_S^{\eta} \big] \E_{ p_{D} } \big{[} g(-1) \big{]}\\
        \
        & = \eta_S \pi_D^{\eta} K C + \eta_D \pi_S^{\eta} K + \frac{(1-\eta_S) \pi_S^{\eta} - \eta_S \pi_D^{\eta} C}{\pi_S}\;\pi_S \E_{ p_{S} } \big{[} g(1) \big{]} + \frac{(1-\eta_D) \pi_D^{\eta} C - \eta_D \pi_S^{\eta}}{\pi_D}\;\pi_D \E_{ p_{D} } \big{[} g(-1) \big{]}\\
    \end{split}
\end{equation*}
Using $C=\frac{\pi_S^{\eta} \big[ (1-\eta_S)\pi_D + \eta_D\pi_S \big]}{\pi_D^{\eta} \big[ \eta_S \pi_D + (1-\eta_D)\pi_S \big] }$, we see that
$\frac{(1-\eta_S) \pi_S^{\eta} - \eta_S \pi_D^{\eta} C}{\pi_S} - \frac{(1-\eta_D) \pi_D^{\eta} C - \eta_D \pi_S^{\eta}}{\pi_D}=
        \frac{(1-\eta_S) \pi_S^{\eta}}{\pi_S} - \frac{\eta_S \pi_D^{\eta} C}{\pi_S} - \frac{(1-\eta_D) \pi_D^{\eta} C}{\pi_D} + \frac{\eta_D \pi_S^{\eta}}{\pi_D}=
        \frac{(1-\eta_S) \pi_S^{\eta}}{\pi_S} + \frac{\eta_D \pi_S^{\eta}}{\pi_D} - \frac{\eta_S \pi_D^{\eta} C}{\pi_S} - \frac{(1-\eta_D) \pi_D^{\eta} C}{\pi_D}=
        \frac{(1-\eta_S) \pi_S^{\eta}}{\pi_S} + \frac{\eta_D \pi_S^{\eta}}{\pi_D} - C \bigg[ \frac{\eta_S \pi_D^{\eta}}{\pi_S} + \frac{(1-\eta_D) \pi_D^{\eta}}{\pi_D} \bigg]=
        \pi_S^{\eta} \bigg[ \frac{(1-\eta_S)}{\pi_S} + \frac{\eta_D}{\pi_D} \bigg] - C \pi_D^{\eta} \bigg[ \frac{\eta_S }{\pi_S} + \frac{(1-\eta_D)}{\pi_D} \bigg]=
        \pi_S^{\eta} \bigg[ \frac{(1-\eta_S)\pi_D + \eta_D\pi_S}{\pi_S\pi_D} \bigg] -C \pi_D^{\eta} \bigg[ \frac{\eta_S \pi_D + (1-\eta_D)\pi_S}{\pi_S \pi_D} \bigg] =0$. This means, $\frac{(1-\eta_S) \pi_S^{\eta} - \eta_S \pi_D^{\eta} C}{\pi_S} = \frac{(1-\eta_D) \pi_D^{\eta} C - \eta_D \pi_S^{\eta}}{\pi_D}$
Using this, we get
\begin{equation*}
    \begin{split}
        &R_{\hat{L}}^{\eta}(f)  = \eta_S \pi_D^{\eta} K C + \eta_D \pi_S^{\eta} K  + \frac{(1-\eta_S) \pi_S^{\eta} - \eta_S \pi_D^{\eta} C}{\pi_S}  \bigg[ \pi_S \E_{ p_{S} } \big{[} g(1) \big{]} + \pi_D \E_{ p_{D} } \big{[} g(-1) \big{]} \bigg]\\
        \
        & = \eta_S \pi_D^{\eta} K C + \eta_D \pi_S^{\eta} C + \frac{(1-\eta_S) \pi_S^{\eta} - \eta_S \pi_D^{\eta} C}{\pi_S}R_{SD}(f)\\
    \end{split}
\end{equation*}
We see that 
\begin{align*}
    \frac{\pi_S^{\eta} \big( 1-\eta_S \big) - \pi_D^{\eta} \eta_S C}{\pi_S}& =\pi_S^{\eta} \frac{( 1-\eta_S) [\eta_S \pi_D + (1-\eta_D)\pi_S] - \eta_S [ (1-\eta_S)\pi_D + \eta_D\pi_S]}{\pi_s(\eta_S\pi_D+(1-\eta_D)\pi_S)}\\
    &=
        \pi_S^{\eta}\frac{\eta_S \big( 1 - \eta_S \big) \pi_D + \big( 1 - \eta_S \big) (1-\eta_D)\pi_S - \eta_S \big( 1 - \eta_S \big) \pi_D - \eta_S \eta_D \pi_S}{\pi_S(\eta_S\pi_D+(1-\eta_D)\pi_S)}\\
        &=\pi_S^{\eta}\frac{(1-\eta_S)(1-\eta_D)\pi_S -\eta_S \eta_D \pi_S}{\pi_S(\eta_S\pi_D+(1-\eta_D)\pi_S)}
        =\pi_S^{\eta}\frac{1-\eta_S-\eta_D}{\eta_S\pi_D+(1-\eta_D)\pi_S}.
\end{align*}
Using the fact that $\eta_S + \eta_D <1$, we
observe that $\frac{\pi_S^{\eta} \big( 1-\eta_S \big) - \pi_D^{\eta} \eta_S C}{\pi_S}>0$. Let $f^*$ be the minimizer of the SD risk $R_{L}$. Which means, $R_{L}(f^*)\leq R_{L}(f),\;\forall f\in \mathcal{F}$. Thus, we see that $\forall f \in \mathcal{F}$
\begin{equation*}
    \begin{split}
        &R_{\hat{L}}^{\eta}(f^{*}) - R_{\hat{L}}^{\eta}(f) = \frac{\pi_S^{\eta} \big( 1-\eta_S \big) - \pi_D^{\eta} \eta_S K}{\pi_S} \bigg{[} R_{L}(f^{*}) - R_{L}(f) \bigg{]}\leq 0.
    \end{split}
\end{equation*}
This shows that $f^*$ also minimized the risk $R_{\hat{L}}$.
\end{proof}
Thus, to achieve robustness in the case of SD conditional SD noise, we need to construct another loss function $\hat{L}$ using loss function $L$ as described in the above theorem. Then, it is shown that a minimizer of the noise-free risk under $L$ also minimizes risk under SD conditional SD label noise using loss $\hat{L}$. Thus, in this case, a robust algorithm has to minimize risk under SD conditional SD label noise using loss $\hat{L}$. This approach also requires that the loss function $l(.,.)$ should satisfy the condition $l(f(\x),1)+l(f(\x),-1)=K,\;\forall \x,\;\forall f$ for some positive constant $K$. Here, $\hat{L}$ requires that we know the noise rates $\eta_S$ and $\eta_D$. Here, we need to estimate the noise rates $\eta_S$ and $\eta_D$.

\subsection{Noise Rate Estimation}
We see that noise rate estimation is not required in the case of uniform SD label noise. But, for the case of SD conditional SD label noise, the proposed approach requires the knowledge of noise rates. This makes noise rate estimation a crucial part of the proposed approach. Below we propose a noise rate estimation approach for the SD conditional noise case. The proposed approach for noise rate estimation is based on the method discussed in \citet{Patrini2017MakingDN}. We modify it for our application. The following theorem is the basis of our noise rate estimation approach. 
\begin{theorem}
Assume a "perfect examples pair" exists for both $S$ and $D$. Which means,
$\exists\ (\bar{\mathbf{x}}_1^{S}, \bar{\mathbf{x}}_2^{S})\in \X\times \X$ such that $p(\bar{\mathbf{x}}_1^{S}, \bar{\mathbf{x}}_2^{S}) > 0$ and   $p(\tau=1 | \bar{\mathbf{x}}_1^{S}, \bar{\mathbf{x}}_2^{S}) = 1 $. Similarly, $\exists\ (\bar{\mathbf{x}}_1^{D}, \bar{\mathbf{x}}_2^{D})\in \X\times \X$ such that $p(\bar{\mathbf{x}}_1^{D}, \bar{\mathbf{x}}_2^{D}) > 0$ and   $p(\tau=-1 | \bar{\mathbf{x}}_1^{D}, \bar{\mathbf{x}}_2^{D}) = 1 $.
Then, it follows that,
$\eta_S=p(\Tilde{\tau}=-1|\bar{\mathbf{x}}_1^S, \bar{\mathbf{x}}_2^S)$ and $\eta_D=p(\Tilde{\tau}=1|\bar{\mathbf{x}}_1^D, \bar{\mathbf{x}}_2^D)$.
\end{theorem}
\begin{proof}
For $ (\bar{\mathbf{x}}_1^S, \bar{\mathbf{x}}_2^S)$, we have $ p(\tau=-1 | \bar{\mathbf{x}}_1^S, \bar{\mathbf{x}}_2^S)=0$ and $p(\tau=1 | \bar{\mathbf{x}}_1^S, \bar{\mathbf{x}}_2^S)=1$. Thus,
\begin{equation*}
        p(\Tilde{\tau}=-1 | \bar{\mathbf{x}}_1^S, \bar{\mathbf{x}}_2^S) 
        = (1-\eta_D) \;p(\tau=-1 | \bar{\mathbf{x}}_1^S, \bar{\mathbf{x}}_2^S) +\eta_S \;p(\tau=1 | \bar{\mathbf{x}}_1^S, \bar{\mathbf{x}}_2^S) =\eta_S.
\end{equation*}
Similarly, for $ (\bar{\mathbf{x}}_1^D, \bar{\mathbf{x}}_2^D)$, we have $p(\tau=1 | \bar{\mathbf{x}}_1^D, \bar{\mathbf{x}}_2^D)=0$ and $p(\tau=-1 | \bar{\mathbf{x}}_1^D, \bar{\mathbf{x}}_2^D)=1$. Thus,
\begin{equation*}
        p(\Tilde{\tau}=1 | \bar{\mathbf{x}}_1^D, \bar{\mathbf{x}}_2^D) = \eta_D\; p(\tau=-1 | \bar{\mathbf{x}}_1^D, \bar{\mathbf{x}}_2^D) +(1-\eta_S) \;p(\tau=1 | \bar{\mathbf{x}}_1^D, \bar{\mathbf{x}}_2^D) =\eta_D.
\end{equation*}
\end{proof}
The primary assumption in the above theorem is that (a) an example pair whose actual SD label is +1 and (b) an example pair whose actual SD label is -1. 

As this noise rate estimation algorithm deals with pairwise data points, we append the features vectors of the pair of examples and make a vector of 2$d$-dimensional. Thus, we create a training set of the following form $\bar{S}_{SD}^{\eta}=\{([\x_1^T\quad \x_2^T],\Tilde{\tau})\;|\; (\x_1,\x_2,\Tilde{\tau})\in S_{SD}^{\eta}\}$. We train a model to learn the mapping $h:\mathbb{R}^{2d}\rightarrow \{-1,+1\}$ which takes $[\x_1^T\quad \x_2^T]^T $ as input and predicts SD label for example pair $(\x_1,\x_2)$ using noisy SD labels. We use cross-entropy loss to train such a network to compare the predicted and noisy SD labels. 

After training, we find the example pair $(\bar{\x}_1^{S},\bar{\x}_2^S)$ for which the network predicts the highest probability of the SD label is 1. Thus, $(\bar{\x}_1^{S},\bar{\x}_2^S)={\arg\max}_{(\x_1,\x_2)\in \X\times\X}\;p(\Tilde{\tau}=1|\x_1,\x_2)$. Using this example pair, we estimate $\hat{\eta}_S=p(\Tilde{\tau}=-1|\bar{\x}_1^{S},\bar{\x}_2^S)$. Similarly, we find the example pair $(\bar{\x}_1^{D},\bar{\x}_2^{D})$ for which the network predicts highest probability of SD label being -1. Thus, $(\bar{\x}_1^{D},\bar{\x}_2^{D})={\arg\max}_{(\x_1,\x_2)\in \X\times\X}\;p(\Tilde{\tau}=-1|\x_1,\x_2)$. 
Using this example pair, we estimate $\hat{\eta}_D=p(\Tilde{\tau}=1|\bar{\x}_1^{D},\bar{\x}_2^{D})$. We use these estimated noise rates in the experiments. Complete details of the noise rate estimation approach are given in Algorithm~\ref{algo:NRE}. 

\begin{algorithm}[H]
    \caption{Noise Rate Estimation Algorithm}
    \label{algo:NRE}
    \begin{algorithmic}
        \STATE \textbf{Input}: Noisy pairwise similarity dissimilarity training set $S_{SD}^{\eta}$
        \STATE Create $\bar{S}_{SD}^{\eta}$ as $\bar{S}_{SD}^{\eta}=\{([\x_1^T\;\;\x_2^T]^T,\Tilde{\tau})\;|\;(\x_1,\x_2,\Tilde{\tau})\in S_{SD}^{\eta}\}$.
        \STATE Partition $\bar{S}_{SD}^{\eta}$ into train ($\bar{S}_{SD(tr)}^{\eta}$) and test ($\bar{S}_{SD(ts)}^{\eta}$) sets in the ratio 80:20
        \STATE Train a network to learn a mapping $h: \R^{2d}\rightarrow \{0,1\}$ using $\bar{S}_{SD(tr)}^{\eta}$ such that $h$ models the probability $p(\Tilde{\tau}=1|\x_1,\x_2)$.
        \STATE Find $(\bar{\x}_1^S,\bar{\x}_1^S)$ using $\bar{S}_{SD(ts)}^{\eta}$ such that $(\bar{\x}_1^S,\bar{\x}_1^S)={\arg\max}_{(\x_1,\x_2)\in \bar{S}_{SD(ts)}^{\eta}}\;p(\tilde{\tau}=1|\x_1,\x_2)$.
        \STATE Find $(\bar{\x}_1^D,\bar{\x}_1^D)$ using $\bar{S}_{SD(ts)}^{\eta}$ such that $(\bar{\x}_1^D,\bar{\x}_1^D)={\arg\max}_{(\x_1,\x_2)\in \bar{S}_{SD(ts)}^{\eta}}\;p(\tilde{\tau}=-1|\x_1,\x_2)$.
        \STATE Find $\hat{\eta}_S$ as $\hat{\eta}_S=p(\Tilde{\tau}=-1|\bar{\x}_1^{S},\bar{\x}_2^S)$.
        \STATE Find $\hat{\eta}_D$ as $\hat{\eta}_D=p(\Tilde{\tau}=1|\bar{\x}_1^{D},\bar{\x}_2^D)$.
        \STATE \textbf{Output}: $\hat{\eta}_S$ and $\hat{\eta}_D$.
    \end{algorithmic}
\end{algorithm}
\section{Analysis of Robustness of Various Loss Functions}
Here we discuss the robustness properties of various loss functions (e.g., 0-1 loss, hinge loss, log loss, squared error loss, ramp loss, sigmoid loss, absolute error loss, etc.)

\begin{enumerate}
    \item We observe that 0-1 loss, ramp loss, sigmoid loss and probit loss functions satisfy the condition $l(f(\x),1)+l(f(\x),-1)=K$ with $K$ values as described in the following table.
    
    \begin{tabular}{||p{0.5in}|p{4in}|p{0.6in}||}
    \hline
    \hline
    Loss  & Functional Form & $K$ value\\
    \hline
    \hline
    0-1  & $l_{0-1}(f(\x),y)=\mathbb{I}[yf(\x)>0]$ where $y\in\{+1,-1\}$, $\mathbb{I}[A]=1$ if $A$ is true and 0 otherwise & $K=1$\\
    \hline
    Ramp  & $l_{ramp}(f(\x),y)=[1-yf(\x)]_+-[-1-yf(\x)]_+$ where $y\in\{+1,-1\}$, $[a]_+=\max(0,a)$ & $K=2$\\
    \hline
    Sigmoid  & $l_{sigmoid}(f(\x),y)=\frac{1}{1+\exp(yf(\x))}$ where $y\in\{+1,-1\}$ & $K=1$\\
    \hline 
    Probit  & $l_{probit}(f(\x),y)=1-\Phi(yf(\x))$ where $y\in\{+1,-1\}$ and $\Phi$ is the cumulative distribution function (CDF) of standard normal distribution & $K=1$\\
    \hline
    Absolute Error & $l_{abs}(f(\x),y)=\vert y-f(\x)\vert$  where $f(\x)$ is sigmoid function and $y\in\{0,1\}$ & $K=1$\\
    \hline
    \hline
    \end{tabular}
    
    Thus, risk minimization under these losses becomes robust to uniform SD label noise and SD conditional SD label noise.
    \item Hinge loss ($l_{hinge}(f(\x),y)=[1-yf(\x)]_+$), squared error loss ($l_{sq}(f(\x),y)=(f(\x)-y)^2$) and log-loss ($l_{log}(f(\x),y)=\log[1+\exp(-yf(\x)]$) do not satisfy the property $l(f(\x),1)+l(f(\x),-1)=K$. Thus, the robustness of risk minimization under these loss functions is not guaranteed in the presence of SD label noise.
\end{enumerate}
\section{Experiments}

In this section, we present experimental results to illustrate the robustness of the proposed approach in the presence of noisy pairwise similarity dis-similarity data. We compare performances of risk minimization with sigmoid loss, ramp loss, mean absolute error (MAE) loss, and mean squared error (MSE) loss with Deep Neural Network (DNN) model architecture. We show experimental results on five real-world binary classification datasets from the UCI ML repository \citep{Dua:2019}. These datasets are Adult, Breast Cancer, Ionosphere, Phishing, and W8a. To show that the proposed approach works efficiently for large-scale image datasets, we show experiments with the CIFAR-10 dataset \citep{cifar-10}. In the CIFAR-10 dataset, we considered two classes (airplane vs. automobile) for the experiments.
\subsection{Creating Noise-free and Noisy Pairwise Similarity Dissimilarity Datasets} 
We randomly choose pair of examples from the original dataset to obtain pairwise data and assign the SD label based on whether they belong to the same class. We call this noise-free pairwise data $S_{SD}$. We induce SD label noise into the pairwise data discussed above to obtain noisy datasets. We flip the SD label based on the noise rate for every pair of examples in the set $S_{SD}$. This noisy data is denoted as $S_{SD}^{\eta}$. We vary the noise rate ($ \eta $) for uniform noise cases from 10\% to 40\%. And for SD conditional SD noise, we provide results on 4 combinations of $ \eta_{S} $ and $ \eta_{D} $. These combinations are (15\%,20\%), (20\%,10\%), (20\%,25\%) and (30\%,25\%).

\subsection{Experimental Setup for The Proposed Approach}
We train a neural network with three hidden layers for each dataset and noise rate using pairwise similar, dissimilar data. For the CIFAR-10 dataset, we use transfer learning to learn the classifier by adding the layers mentioned above to the pre-trained VGG-19 model. The best architecture of the classifier we used for experiments is a three-layer deep neural network with rectified linear unit (ReLU) activation function between any two consecutive layers. The number of nodes in the first, second, and third hidden layers is 128, 32, and 8, respectively. The output layer uses the sigmoid activation function. We trained with a mini-batch of size 256. We used simple gradient descent with decreasing step size and a momentum term (Adam optimizer). We used PyTorch \citep{NEURIPS2019_9015} library for deep learning. For Adam optimizer, we used the default values (learning rate =0.001, $\beta_1=0.9,\;\beta_2= 0.999, \epsilon=10^{-8}$, weight decay=0).

We train neural networks with different loss functions (sigmoid, ramp, and absolute error loss). The risk function is a modification of the binary classification risk function that accommodates the pairwise classification problem. We used a 10-fold cross-validation approach to generate all the results. The results reported are averaged over ten such independent runs. We report results for each setting with standard deviation.

\subsection{Experiments with Noise Rate Estimation}
As discussed earlier, noise rate estimation is required for SD conditional SD label noise. For the noise rate estimation classifier, we use one hidden layer (with eight nodes) neural network with a dropout layer between the hidden layer and output layer. We used rectified linear unit (ReLU) activation function in the hidden layer. We used the cross-entropy loss function with the Adam optimizer. We randomly used 80\% for training and 20\% for testing sets. The training set is corrupted with the required amount of class conditional label noise before being fed into the classifier for training. We obtain the estimated noise rates using the test set from the learnt classifier. We report estimated noise rates for each dataset obtained by averaging over ten independent runs. Noise rate estimation results are reported in Table~\ref{results-nre} for various datasets and noise rates. We observe that the noise rates estimated by the proposed noise rate estimation approach are very close to the actual noise rates.

\begin{table}[h]
  \caption{Experimental Results on Noise Rate Estimation}
  \label{results-nre}
  \centering
  \begin{tabular}{llllll}
    \toprule
    & Adult & Breast Cancer & Ionosphere & Phishing & W8a\\
    \midrule
     Original & Estimated & Estimated & Estimated & Estimated & Estimated \\
     $ (\eta_{S},\eta_{D}) $  & $ (\hat{\eta}_{S},\hat{\eta}_{D}) $ & $ (\hat{\eta}_{S},\hat{\eta}_{D}) $ & $ (\hat{\eta}_{S},\hat{\eta}_{D}) $ & $ (\hat{\eta}_{S},\hat{\eta}_{D}) $ & $ (\hat{\eta}_{S},\hat{\eta}_{D})$ \\
     in \% & in \% & in \% & in \% & in \% & in \%\\
     \midrule
     (15,20) & (14.77,  22.52) & (11.50, 18.29) & (17.00, 19.65) & (11.39, 18.34) & (11.91, 16.17) \\
     (20, 10) & (20.81,13.81) & (16.55, 07.77) & (21.83, 09.71) & (15.84, 13.24) & (15.25, 07.71)\\
     (20, 25) & (19.55,24.62) & (16.58, 25.12) & (22.69, 25.40) & (16.59, 25.07) & (16.66, 21.07) \\
     (30, 25) & (26.38, 25.48)  & (28.11, 23.73) & (33.12, 24.75) & (26.20, 25.17) & (26.51, 22.39)\\
    \bottomrule
  \end{tabular}
\end{table}

\subsection{Baseline Used}
We used a loss correction-based approach \citet{10.1007/978-3-030-86520-7_15} for learning from noisy similar and dissimilar data as a baseline. We use the acronym LCNiP (Loss Correction based approach for learning using NoIsy Pairwise data) for this baseline. The authors have experimented only with this work's squared error loss function. So, we also generate results with this approach for the squared error loss function. We used the proposed method's network architecture and experimental setting to make a fair comparison. Note that this approach also requires noise rate as an input. We used the noise rates estimated by the proposed noise rate estimation approach. For this approach also, we used 10-fold cross-validation to generate all the results. The results reported are averaged over ten such independent runs along with standard deviation.

\begin{table}[h]
  \caption{Experimental Results With Uniform SD Label Noise}
  \label{results-un}
  \centering
  \begin{tabular}{|l|l|lll|l|}
    \toprule
    & & \multicolumn{3}{c|}{RoLNiP} & \multicolumn{1}{c|}{LCNiP}  \\
    \cmidrule(r){3-5}
    \cmidrule(r){6-6}
     & Noise & Sigmoid Loss & Ramp Loss & Absolute &  Squared\\
     & Rate ($ \eta $) & & & Error Loss & Error Loss \\
     \midrule
     & 0\% & $ 86.35 \pm 0.09 $ & $ {\bf 87.16} \pm 0.07 $ & $ 86.37 \pm 0.05 $ &  $ 87.02 \pm 0.13 $ \\
     & 10\% & $ 85.87 \pm 0.12 $ & $ {\bf 86.25} \pm 0.10 $ & $ 85.83 \pm 0.06 $  & $ 84.77 \pm 0.75 $ \\
     Adult & 20\% & $ 85.33 \pm 0.05 $ & $ {\bf 85.46} \pm 0.22 $ & $ 85.37 \pm 0.08 $  & $ 82.22 \pm 0.33 $ \\
     & 30\% & $ 84.30 \pm 0.20 $ & $ 84.25 \pm 0.25 $ & $ {\bf 84.34} \pm 0.13 $ &  $ 77.15 \pm 0.62 $\\
     & 40\% & $ 81.48 \pm 0.36 $ & $ 79.61 \pm 1.07 $ & $ {\bf 81.95} \pm 0.27 $  & $ 65.58 \pm 1.89 $\\
    \midrule
     & 0\% & $ 98.93 \pm 0.05 $ & $ {\bf 99.01} \pm 0.22 $ & $ 98.86 \pm 0.07 $ &  $ 98.67 \pm 0.09 $ \\
     & 10\% & $ 98.65 \pm 0.03 $ & $ {\bf 98.74} \pm 0.02 $ & $ 98.65 \pm 0.07 $  & $ 98.47 \pm 0.86 $\\
     Breast & 20\% & $ {\bf 98.46} \pm 0.03 $ & $ 98.38 \pm 0.03 $ & $ 98.44 \pm 0.02 $ & $ 97.95 \pm 1.17 $\\
     Cancer & 30\% & $ 98.44 \pm 0.04 $ & $ {\bf 98.48} \pm 0.03 $ & $ {\bf 98.44} \pm 0.04 $  & $ 96.09 \pm 1.34 $ \\
     & 40\% & $ 97.70 \pm 0.28 $ & $ {\bf 98.22} \pm 0.28 $ & $ 97.80 \pm 0.21 $  & $ 80.18 \pm 1.63 $ \\
    \midrule
     & 0\% & $ 98.40 \pm 0.14 $ & $ 98.64 \pm 0.58 $ & $ 98.76 \pm 0.16 $  & $ {\bf 99.12} \pm 0.21 $ \\
     & 10\% & $ 98.58 \pm 0.15 $ & $ 98.76 \pm 0.40 $ & $ 98.53 \pm 0.17 $ &  $ {\bf 98.84} \pm 0.84 $ \\
     Ionosphere & 20\% & $ 97.99 \pm 0.16 $ & $ 98.09 \pm 0.63 $ & $ {\bf 98.59} \pm 0.20 $ & $ 98.02 \pm 0.22 $ \\
     & 30\% & $ 96.12 \pm 0.17 $ & $ 96.31 \pm 0.77 $ & $ {\bf 98.20} \pm 0.22 $ & $ 95.73 \pm 0.76 $ \\
     & 40\% & $ 91.04 \pm 0.26 $ & $ 90.62 \pm 0.89 $ & $ {\bf 95.04} \pm 0.28 $  & $ 80.08 \pm 3.07 $ \\
    \midrule
     & 0\% & $ 94.32 \pm 0.92 $ & $ {\bf 95.35} \pm 0.12 $ & $ 94.00 \pm 1.15 $ &  $ 94.28 \pm 0.81 $ \\
     & 10\% & $ 93.58 \pm 1.28 $ & $ {\bf 95.00} \pm 0.11 $ & $ 94.59 \pm 0.04 $  & $ 94.15 \pm 1.02 $ \\
     Phishing & 20\% & $ 93.70 \pm 1.09 $ & $ {\bf 94.73} \pm 0.12 $ & $ 93.66 \pm 1.16 $  & $ 93.34 \pm 1.61 $ \\
     & 30\% & $ 93.80 \pm 1.44 $ & $ 91.78 \pm 7.24 $ & $ {\bf 94.23} \pm 0.10 $ & $ 92.12 \pm 0.84 $ \\
     & 40\% & $ {\bf 90.00} \pm 0.28 $ & $ 87.07 \pm 5.68 $ & $ 89.76 \pm 0.62 $ &  $ 82.71 \pm 2.30 $ \\
    \midrule
     & 0\% & $ {\bf 98.08} \pm 0.05 $ & $ 97.96 \pm 0.21 $ & $ 98.05 \pm 0.06 $ &  $ 93.68 \pm 0.35 $ \\
     & 10\% & $ {\bf 97.59} \pm 0.09 $ & $ 97.29 \pm 0.41 $ & $ {\bf 97.59} \pm 0.09 $ & $ 93.28 \pm 0.36 $ \\
     W8a & 20\% & $ {\bf 96.68} \pm 0.08 $ & $ 96.12 \pm 0.28 $ & $ 96.63 \pm 0.13 $ &  $ 90.56 \pm 0.50 $ \\
     & 30\% & $ {\bf 94.62} \pm 0.16 $ & $ 93.00 \pm 0.58 $ & $ 94.51 \pm 0.13 $ &  $ 87.12 \pm 1.24 $ \\
     & 40\% & $ {\bf 86.87} \pm 0.23 $ & $ 85.01 \pm 1.10 $ & $ 86.82 \pm 0.28 $ &  $ 73.35 \pm 2.78 $ \\
     \midrule
     & 0\% & $ 92.24 \pm 0.99 $ & $ 92.10 \pm 1.10 $ & $ 91.96 \pm 0.52 $ & $ {\bf 92.43} \pm 0.85 $ \\
     & 10\% & $ {\bf 92.10} \pm 0.65 $ & $ 91.46 \pm 0.73 $ & $ 90.91 \pm 1.24 $ & $ 91.80 \pm 0.86 $ \\
     CIFAR-10 & 20\% & $ {\bf 91.10} \pm 0.86 $ & $ 90.77 \pm 0.60 $ & $ 90.45 \pm 1.27 $ & $ 90.53 \pm 0.77 $ \\
     & 30\% & $ {\bf 89.73} \pm 1.15 $ & $ 88.03 \pm 1.73 $ & $ 88.38 \pm 1.79 $ & $ 89.54 \pm 1.56 $ \\
     & 40\% & $ 84.65 \pm 1.76 $ & $ 84.78 \pm 1.94 $ & $ {\bf 85.68} \pm 2.06 $ & $ 84.05 \pm 1.33 $ \\ 
    \bottomrule
  \end{tabular}
\end{table}

\begin{table}[t]
  \caption{Experimental Results with SD Conditional SD Label Noise (Known Noise Rates)}
  \label{results-cc}
  \centering
  \begin{tabular}{|l|ll|lll|l|}
    \toprule
    & & & \multicolumn{3}{c|}{RoLNiP} & \multicolumn{1}{c|}{LCNiP}  \\
    \cmidrule(r){4-6}
    \cmidrule(r){7-7}
    & \multicolumn{2}{c|}{Noise Rates} & Sigmoid Loss & Ramp Loss & Absolute & Squared\\
     & $ \eta_{S} $ & $ \eta_{D} $ & & & Error Loss & Error Loss \\
     \midrule
     & 0\% & 0\% & $ 86.35 \pm 0.09 $ & $ {\bf 87.16} \pm 0.07 $ & $ 86.37 \pm 0.05 $ &  $ 87.02 \pm 0.13 $ \\
     & 15\% & 20\% & $ 84.79 \pm 0.10 $ & $ {\bf 84.90} \pm 0.11 $ & $ 84.84 \pm 0.13 $ & $ 82.94 \pm 0.91 $ \\
     Adult & 20\% & 10\% & $ 84.05 \pm 0.53 $ & $ 84.03 \pm 0.54 $ & $ {\bf 84.17} \pm 0.11 $ & $ 78.54 \pm 1.30 $ \\
     & 20\% & 25\% & $ 84.26 \pm 0.15 $ & $ 84.33 \pm 0.14 $ & $ {\bf 84.38} \pm 0.20 $ & $ 81.63 \pm 0.47 $ \\
     & 30\% & 25\% & $ 83.40 \pm 0.39 $ & $ 83.49 \pm 0.13 $ & $ {\bf 83.56} \pm 0.10 $ & $ 73.89 \pm 1.35 $ \\
    \midrule
     & 0\% & 0\% & $ 98.93 \pm 0.05 $ & $ {\bf 99.01} \pm 0.22 $ & $ 98.86 \pm 0.07 $ &  $ 98.67 \pm 0.09 $ \\
     & 15\% & 20\% & $ {\bf 98.63} \pm 0.05 $ & $ 98.60 \pm 0.04 $ & $ {\bf 98.63} \pm 0.06 $ & $ 98.49 \pm 0.09 $ \\
     Breast & 20\% & 10\% & $ {\bf 98.70} \pm 0.08 $ & $ 98.65 \pm 0.04 $ & $ 98.65 \pm 0.04 $ & $ 95.92 \pm 1.36 $ \\
     Cancer & 20\% & 25\% & $ {\bf 98.61} \pm 0.05 $ & $ 98.60 \pm 0.04 $ & $ 98.58 \pm 0.04 $ & $ 98.08 \pm 0.86 $ \\
     & 30\% & 25\% & $ 98.50 \pm 0.08 $ & $ 98.53 \pm 0.06 $ & $ {\bf 98.56} \pm 0.06 $ & $ 93.09 \pm 1.54 $ \\
    \midrule
     & 0\% & 0\% & $ 98.40 \pm 0.14 $ & $ 98.64 \pm 0.58 $ & $ 98.76 \pm 0.16 $  & $ {\bf 99.12} \pm 0.21 $ \\
     & 15\% & 20\% & $ 98.64 \pm 0.10 $ & $ 98.60 \pm 0.14 $ & $ {\bf 98.65} \pm 0.07 $ & $ 97.83 \pm 0.26 $ \\
     Ionosphere & 20\% & 10\% & $ 97.19 \pm 0.35 $ & $ 97.11 \pm 0.20 $ & $ 97.14 \pm 0.21 $ & $ {\bf 97.62} \pm 0.59 $ \\
     & 20\% & 25\% & $ {\bf 98.56} \pm 0.06 $ & $ 98.48 \pm 0.07 $ & $ 98.51 \pm 0.08 $ & $ 96.62 \pm 0.30 $ \\
     & 30\% & 25\% & $ 95.94 \pm 0.26 $ & $ 95.90 \pm 0.25 $ & $ {\bf 96.02} \pm 0.28 $ & $ 94.41 \pm 2.02 $ \\
    \midrule
     & 0\% & 0\% & $ 94.32 \pm 0.92 $ & $ {\bf 95.35} \pm 0.12 $ & $ 94.00 \pm 1.15 $ &  $ 94.28 \pm 0.81 $ \\
     & 15\% & 20\% & $ 94.38 \pm 0.07 $ & $ 94.40 \pm 0.10 $ & $ {\bf 94.42} \pm 0.07 $ & $ 93.48 \pm 0.94 $ \\
     Phishing & 20\% & 10\% & $ 93.30 \pm 1.20 $ & $ {\bf 93.75} \pm 0.26 $ & $ 93.04 \pm 2.03 $ & $ 90.65 \pm 2.26 $ \\
     & 20\% & 25\% & $ 94.08 \pm 0.19 $ & $ {\bf 94.22} \pm 0.13 $ & $ 93.98 \pm 0.76 $ & $ 92.66 \pm 1.56 $ \\
     & 30\% & 25\% & $ 93.15 \pm 1.63 $ & $ {\bf 93.54} \pm 0.26 $ & $ 93.51 \pm 1.16 $ & $ 89.45 \pm 1.99 $ \\
    \midrule
     & 0\% & 0\% & $ {\bf 98.08} \pm 0.05 $ & $ 97.96 \pm 0.21 $ & $ 98.05 \pm 0.06 $ &  $ 93.68 \pm 0.35 $ \\
     & 15\% & 20\% & $ 96.66 \pm 0.09 $ & $ {\bf 96.69} \pm 0.10 $ & $ 96.62 \pm 0.07 $ & $ 89.82 \pm 1.04 $ \\
     W8a & 20\% & 10\% & $ {\bf 96.93} \pm 0.12 $ & $ 96.75 \pm 0.15 $ & $ 96.80 \pm 0.11 $ & $ 93.63 \pm 0.89 $ \\
     & 20\% & 25\% & $ 95.90 \pm 0.11 $ & $ 95.88 \pm 0.09 $ & $ {\bf 95.97} \pm 0.15 $ & $ 88.86 \pm 0.62 $ \\
     & 30\% & 25\% & $ {\bf 94.38} \pm 0.09 $ & $ 94.18 \pm 0.22 $ & $ 94.15 \pm 0.20 $ & $ 87.21 \pm 1.10 $ \\
     \midrule
     & 0\% & 0\% & $ 92.24 \pm 0.99 $ & $ 92.10 \pm 1.10 $ & $ 91.96 \pm 0.52 $ & $ {\bf 92.43} \pm 0.85 $ \\
     & 15\% & 20\% & $ {\bf 90.80} \pm 1.04 $ & $ 90.53 \pm 2.40 $ & $ 90.27 \pm 1.32 $ & $ 90.72 \pm 3.97 $ \\
     CIFAR-10 & 20\% & 10\% & $ 90.99 \pm 0.81 $ & $ 91.10 \pm 0.98 $ & $ 90.27 \pm 1.21 $ & $ {\bf 92.96} \pm 0.75 $ \\
     & 20\% & 25\% & $ 89.55 \pm 1.10 $ & $ {\bf 90.01} \pm 1.15 $ & $ 89.75 \pm 1.32 $ & $ 89.09 \pm 2.28 $ \\
     & 30\% & 25\% & $ 89.82 \pm 0.98 $ & $ 89.71 \pm 1.33 $ & $ 89.75 \pm 1.33 $ & $ {\bf 91.26} \pm 0.86 $ \\
    \bottomrule
  \end{tabular}
\end{table}

\begin{table}[t]
  \caption{Experimental Results with SD Conditional SD Label Noise (Estimated Noise Rates)}
  \label{results-cce}
  \centering
  \begin{tabular}{|l|ll|lll|l|}
    \toprule
    & & & \multicolumn{3}{c|}{RoLNiP} & \multicolumn{1}{c|}{LCNiP}  \\
    \cmidrule(r){4-6}
    \cmidrule(r){7-7}
    & \multicolumn{2}{c|}{Noise Rates} & Sigmoid Loss & Ramp Loss & Absolute & Squared\\
     & $ \eta_{S} $ & $ \eta_{D} $ & & & Error Loss & Error Loss \\
     \midrule
     & 0\% & 0\% & $ 86.35 \pm 0.09 $ & $ {\bf 87.16} \pm 0.07 $ & $ 86.37 \pm 0.05 $ &  $ 87.02 \pm 0.13 $ \\
     & 15\% & 20\% & $ {\bf 85.31} \pm 0.09 $ & $ 83.98 \pm 0.33 $ & $ 84.20 \pm 0.14 $ & $ 81.72 \pm 0.54 $ \\
     Adult & 20\% & 10\% & $ 85.29 \pm 0.08 $ & $ 85.17 \pm 0.11 $ & $ {\bf 85.40} \pm 0.06 $ & $ 75.88 \pm 1.01 $ \\
     & 20\% & 25\% & $ 84.84 \pm 0.12 $ & $ {\bf 84.9} \pm 0.29 $ & $ 84.68 \pm 0.14 $ & $ 80.08 \pm 0.50 $ \\
     & 30\% & 25\% & $ {\bf 84.00} \pm 0.13 $ & $ 81.8 \pm 0.64 $ & $ 83.40 \pm 0.20 $ & $ 69.54 \pm 1.95 $ \\
    \midrule
    & 0\% & 0\% & $ 98.93 \pm 0.05 $ & $ {\bf 99.01} \pm 0.22 $ & $ 98.86 \pm 0.07 $ &  $ 98.67 \pm 0.09 $ \\
     & 15\% & 20\% & $ {\bf 98.64} \pm 0.03 $ & $ 97.98 \pm 0.09 $ & $ 98.28 \pm 0.17 $ & $ 98.43 \pm 0.14 $ \\
     Breast & 20\% & 10\% & $ 98.73 \pm 0.03 $ & $ {\bf 99.02} \pm 0.02 $ & $ 98.71 \pm 0.04 $ & $ 96.99 \pm 0.63 $ \\
     Cancer & 20\% & 25\% & $ {\bf 98.64} \pm 0.04 $ & $ 97.80 \pm 2.02 $ & $ 98.62 \pm 0.03 $ & $ 98.03 \pm 0.92 $ \\
     & 30\% & 25\% & $ 98.57 \pm 0.09 $ & $ {\bf 98.91} \pm 0.07 $ & $ 98.57 \pm 0.06 $ & $ 93.80 \pm 1.46 $ \\
    \midrule
     & 0\% & 0\% & $ 98.40 \pm 0.14 $ & $ 98.64 \pm 0.58 $ & $ 98.76 \pm 0.16 $  & $ {\bf 99.12} \pm 0.21 $ \\
     & 15\% & 20\% & $ 96.70 \pm 0.32 $ & $ {\bf 98.49} \pm 0.17 $ & $ 98.24 \pm 0.16 $ & $ 97.67 \pm 0.81 $ \\
     Ionosphere & 20\% & 10\% & $ 97.43 \pm 0.19 $ & $ {\bf 97.74} \pm 0.12 $ & $ 94.50 \pm 0.32 $ & $ 97.57 \pm 0.71 $ \\
     & 20\% & 25\% & $ 98.52 \pm 0.05 $ & $ 97.79 \pm 0.20 $ & $ {\bf 98.57} \pm 0.07 $ & $ 96.29 \pm 0.78 $ \\
     & 30\% & 25\% & $ 95.06 \pm 0.24 $ & $ 94.47 \pm 0.30 $ & $ {\bf 97.32} \pm 0.15 $ & $ 95.52 \pm 1.10 $ \\
    \midrule
     & 0\% & 0\% & $ 94.32 \pm 0.92 $ & $ {\bf 95.35} \pm 0.12 $ & $ 94.00 \pm 1.15 $ &  $ 94.28 \pm 0.81 $ \\
     & 15\% & 20\% & $ 93.11 \pm 1.88 $ & $ {\bf 93.95} \pm 0.92 $ & $ 93.92 \pm 0.69 $ & $ 93.26 \pm 1.50 $ \\
     Phishing & 20\% & 10\% & $ {\bf 94.20} \pm 0.76 $ & $ 93.88 \pm 1.65 $ & $ 93.60 \pm 1.59 $ & $ 91.62 \pm 1.74 $ \\
     & 20\% & 25\% & $ 89.79 \pm 3.13 $ & $ {\bf 93.95} \pm 0.18 $ & $ 93.35 \pm 1.63 $ & $ 92.96 \pm 1.11 $ \\
     & 30\% & 25\% & $ 93.77 \pm 0.20 $ & $ 93.87 \pm 0.65 $ & $ {\bf 94.73} \pm 0.11 $ & $ 89.84 \pm 2.51 $ \\
    \midrule
     & 0\% & 0\% & $ {\bf 98.08} \pm 0.05 $ & $ 97.96 \pm 0.21 $ & $ 98.05 \pm 0.06 $ &  $ 93.68 \pm 0.35 $ \\
     & 15\% & 20\% & $ {\bf 96.84} \pm 0.09 $ & $ 95.69 \pm 0.49 $ & $ 96.78 \pm 0.07 $ & $ 89.94 \pm 0.60 $ \\
     W8a & 20\% & 10\% & $ 96.71 \pm 0.12 $ & $ {\bf 98.09} \pm 0.15 $ & $ 96.96 \pm 0.08 $ & $ 93.57 \pm 1.19 $ \\
     & 20\% & 25\% & $ 96.03 \pm 0.06 $ & $ 95.23 \pm 0.29 $ & $ {\bf 96.18} \pm 0.20 $ & $ 89.19 \pm 0.55 $ \\
     & 30\% & 25\% & $ 95.01 \pm 0.22 $ & $ {\bf 96.83} \pm 0.19 $ & $ 93.65 \pm 0.16 $ & $ 87.55 \pm 0.87 $ \\
    \bottomrule
  \end{tabular}
\end{table}

\subsection{Results with Uniform SD Label Noise}
For uniform SD label noise, the proposed approach RoLNiP does not require noise rate as an input. However, the loss correction-based approach LCNiP \citep{10.1007/978-3-030-86520-7_15} requires noise rate as input even for uniform noise cases. We used actual noise rates for LCNiP. The results for the uniform noise case are presented in Table~\ref{results-un}. We make the following observations.
\begin{itemize}
    \item RoLNiP Versus LCNiP: RoLNiP with sigmoid, ramp, and absolute loss outperforms the baseline method LCNiP on all the datasets and all the noise rates except for the Ionosphere dataset with 10\% noise rate. For higher noise rates (30\% and 40\%), the performance of LCNiP degrades a lot. In the 40\% noise rate, the accuracy of LCNiP drops by (a) 21.5\% for Adult dataset, (b) 18.5\% for Breast Cancer dataset, (c) 19\% for Ionosphere dataset, (d) 11.5\% for Phishing dataset and (e) 20\% for W8a dataset. 
    \item RoLNiP with Sigmoid Loss: It shows robustness against uniform SD label noise. This also confirms the applicability of Theorem~1 as sigmoid loss function satisfies the condition $l(f(\x),1)+l(f(\x),-1)=K$. In the case of 40\% noise, accuracy drops by (a) 5\% for Adult dataset, (b) 1\% for Breast Cancer dataset, (c) 7\% for Ionosphere dataset, (d) 4.3\% for Phishing dataset, (e) 11\% for W8a dataset and (f) 7.6\% for CIFAR-10.
    \item RoLNip with Ramp Loss: Ramp loss satisfies the symmetry condition required for robustness. The robustness of ramp loss is also reflected in the experiments. At 40\% noise, accuracy of RolNip drops by (a) 7.5\% for Adult dataset, (b) 0.73\% for Breast Cancer dataset, (c) 9\% for Ionosphere dataset, (d) 8\% for Phishing dataset, (e) 13\% for W8a dataset and (f) 7.3\% for CIFAR-10 dataset.
    \item RoLNiP with Absolute Error Loss: It shows robustness even though absolute error loss does not satisfy the robustness conditions mentioned in Theorem~1. This happens because those conditions are only sufficiency conditions. We see that at a 40\% noise rate, the accuracy drops by at least (a) 4.3\% for Adult dataset, (b) 1\% for Breast Cancer dataset, (c) 3.7\% for Ionosphere dataset, (d) 4.2\% for Phishing dataset, (e) 11.2\% for W8a dataset and (f) 6.3\% for CIFAR-10.
\end{itemize}

\subsection{Results with SD Conditional SD Label Noise}
Table~\ref{results-cc} shows the known noise rates, and Table~\ref{results-cce} shows the results with estimated noise rates. We make the following observations from the results.
\begin{itemize}
    \item RoLNiP Versus LCNiP: We observe that with actual as well as estimated noise rates, RoLNiP outperforms LCNiP for Adult, Breast Cancer, Phishing, and W8a datasets for all combinations of noise rates. For the Ionosphere dataset, RoLNiP outperforms LCNiP for all noise combinations except the (20\%,10\%) case. For the (20\%,10\%) case, the performance of RoLNiP is comparable to LCNiP. For the CIFAR-10 dataset, RolNiP and LCNiP perform comparably to each other except for the (30\%,25\%) noise case.
    \item RoLNiP with Sigmoid Loss, Ramp Loss and Absolute Error Loss: We observe that sigmoid loss, ramp loss and absolute error loss show robust behavior with SD conditional SD label noise.
\end{itemize} 
Thus, overall, the proposed approach RoLNiP outperforms the state-of-the-art LCNiP \citep{10.1007/978-3-030-86520-7_15} on different datasets and various noise rates.

\section{Conclusion}
In this paper, we proposed a robust approach for learning using noisy pairwise similar dissimilar data. We presented conditions under which risk minimization becomes robust to noise. We show that the proposed method is robust to uniform and conditional noise cases. The proposed approach does not require the noise rate to be known for uniform noise; however, the conditional noise case needs noise rate as input. We offer a provably correct approach for learning noise rates to deal with this situation. We experimentally show that the proposed method performs better than the state-of-the-art approaches in various noise settings and with multiple datasets.

\bibliography{example_paper}

\end{document}